\documentclass[accepted]{article}

\usepackage{times}
\usepackage{graphicx} 
\usepackage{tabularx}
\usepackage{subfig}
\usepackage{hyperref}
\usepackage{url}
\usepackage{amsmath}
\usepackage{amsthm}
\usepackage{caption}
\usepackage{float}
\usepackage{amssymb}
\usepackage[colorinlistoftodos,prependcaption]{todonotes}
\usepackage{grffile}
\newtheorem{theorem}{Theorem} 
\newtheorem{definition}{Definition}
\newtheorem{lemma}{Lemma}
\newtheorem{remark}{Remark}
\newtheorem{example}{Example}

\newtheorem{corollary}{Corollary}

\newcommand\independent{\protect\mathpalette{\protect\independenT}{\perp}}
\def\independenT#1#2{\mathrel{\rlap{$#1#2$}\mkern2mu{#1#2}}}

\input xy
\xyoption{all}

\usepackage{natbib}

\usepackage{algorithm}
\usepackage{algorithmic}

\usepackage{hyperref}

\usepackage{icml2016} 

\icmltitlerunning{A New Measure of Conditional Dependence }

\begin{document} 

\twocolumn[
\icmltitle{A New Measure of Conditional Dependence }


\icmlauthor{Jalal Etesami}{etesami2@illinois.edu}
\icmladdress{Department of ISE, Coordinated Science Laboratory,\\
			University of Illinois at Urbana-Champaign Urbana, IL 61801 USA}
\icmlauthor{Kun Zhang}{kunz1@andrew.cmu.edu}
\icmladdress{Department of Philosophy, \\
Carnegie Mellon University, Pittsburgh, PA 15213 USA }
\icmlauthor{Negar Kiyavash}{kiyavash@illinois.edu}
\icmladdress{Department of ECE and ISE, Coordinated Science Laboratory,\\
			University of Illinois at Urbana-Champaign Urbana, IL 61801 USA}


\vskip 0.3in
]

\begin{abstract}
Measuring conditional dependencies among the variables of a network is of great interest to many disciplines. 
This paper studies some shortcomings of the existing dependency measures in detecting direct causal influences or their lack of ability for group selection to capture strong dependencies and accordingly introduces a new statistical dependency measure to overcome them. 
This measure is inspired by Dobrushin's coefficients and based on the fact that there is no dependency between $X$ and $Y$ given another variable $Z$, if and only if the conditional distribution of $Y$ given $X=x$ and $Z=z$ does not change when $X$ takes another realization $x'$ while $Z$ takes the same realization $z$.
We show the advantages of this measure over the related measures in the literature. 
Moreover, we establish the connection between our measure and the integral probability metric (IPM) that helps to develop estimators of the measure with lower complexity compared to other relevant information theoretic based measures.
Finally, we show the performance of this measure through numerical simulations.
\end{abstract}

\section{Introduction}

Identifying the conditional independencies (CIs) among the variables or processes in a systems is a fundamental problem in scientific investigations in different fields such as biology, econometric, social sciences, and many others.

  

In probability theory, two events $X$ and $Y$ are conditionally independent given a third event $Z$, if the occurrence or non-occurrence of $X$ and $Y$ are ``independent" events in their conditional probability distribution given $Z$ \cite{gorodetskii1978strong}.
There are several CI measures in literature that have been developed for different applications to capture such independency. For instance, the most commonly used one is conditional mutual information (CMI) \cite{gorodetskii1978strong} that is an information theoretical quantity. This measure has been used in different fields such as communication engineering, channel coding \cite{cover2012elements}, and causal discovery \cite{spirtes2000causation}.
CMI between $X$ and $Y$ given $Z$ is defined by comparing two conditional distributions: $P(X|Y,Z)$ and $P(X|Z)$ using KL-divergence and then taking average over the conditioning variable $Z$. Hence, it is limited to those realizations with positive probability (see Section \ref{sec:info}).
One shortcoming of such measure is that it cannot capture CIs that occur rarely or even over zero measure sets. Another shortcoming of this measure is that it is symmetric and thus it fails to encode asymmetric dependencies such as causal directions in a network.

Most of the conditional dependency/independency measures are defined similar to the CMI in a sense that they take average over the conditioning variables. Kernel-based method in \cite{ZhangPJS2011} is another example.
Consequently, such measures may fail to distinguish the range of the conditioning variable $Z$ in which the dependency between the variables of interest $X$ and $Y$ is more clearer.
For example, consider a treatment that has different effects on a special disease for different genders. 
There are scenarios in which the previous CI measures (e.g., CMI) fail to identify for which gender the effect of the treatment on the disease is maximized (see Section \ref{sec:cc}).

Discovering the causal relationships in a network is one of the main applications for CI measures \cite{spirtes2000causation}. In this area, it is important to capture the direct causal influence between two variables in a network independent of the other causal indirect influences between them. As we will show in Section \ref{sec:direct}, previous CI measures (e.g., CMI) cannot capture the direct causal influences between two variables (cause and effect) in a network when some variables in the indirect causal path depend on the cause almost deterministically.

The main contribution of this paper is the introduction of a statistical metric inspired by Dobrushin's coefficient \cite{dobrushin1970prescribing} to measure the dependency/independency between $X$ and $Y$ given $Z$ in a network from their realizations.
Our metric has been developed based on the paradigm that if $Y$ has no dependency on $X$ given $Z$, then the conditional distribution of $Y$ given $X=x$ and $Z=z$ will not change if $x$ varies and $Z$ takes the same realization $z$. We will show that this dependency measure overcomes the aforementioned limitations. Moreover, we will establish the connection between our meausre and the IPM to develop estimators for our metric with lower complexity compared to other relevant information-theoretic based measures such as CMI. This is because the proposed estimators depend on the sample points only through the metric of the space, and thus its complexity is independent of the dimension of the samples.

Perhaps the best known paradigm for visualizing the CIs among the variables of a network is Bayesian networks \cite{pearl2003causality}. They are directed acyclic graphs (DAGs) in which nodes represent random variables and directed edges denote the direction of causal influences. 
Analogously, using the dependency measure in this work, we can represent the causal structure of a network via a DAG that possesses the same properties as the Bayesian networks.

It is also worth mentioning that there exist several measures to capture CIs and the causal influences among time series, for instance, 
transfer entropy \cite{schreiber2000measuring} and directed information \cite{massey1990causality}. Measuring the reduction of uncertainty in one variable after knowing another variable is the key idea in such measures. 
Because these measure are defined based on CMI, they also suffer the aforementioned limitations.
Note that the proposed measure can easily be modified to capture such influences in time series as well.

\vspace{-.3cm}

\section{Definitions}\label{sec:pre}
\vspace{-.2cm}

In this Section, we review some basic definitions and our notation. 
Throughout this paper we use capital letters to represent random variables, lowercase letters to denote a realization of a random variable, and bold capital letters to denote matrices. 
We denote a subset of random variables with index set $\mathcal{K}\subseteq[m]$, where $[m]:=\{1,...,m\}$ by $\underline{X}_{\mathcal{K}}$ and $[m]\setminus\{j\}$ by $-\{j\}$. 

In a directed graph $\overrightarrow{G}=(V,\overrightarrow{E})$, we denote the parent set of a node $i\in V$ by $Pa_i:=\{j: (j,i)\in\overrightarrow{E}\}$, and denote the set of its non-descendant\footnote{A node $v$ is a non-descendant of another node $u$, if there is no direct path from $u$ to $v$.} by $Nd_i$. 
We use $X\independent Y|Z$ to denote $X$ and $Y$ are independent given $Z$.

\textbf{Bayesian Network}:\ 
A Bayesian network is a graphical model that represents the conditional independencies among a set of random variables via a directed acyclic graph (DAG) \cite{spirtes2000causation}.
A set of random variables $\underline{X}$ is Bayesian with respect to a DAG $\overrightarrow{G}$, if
\vspace{-.1cm}
\begin{small}
\begin{align}\label{fact}
P(\underline{X})=\prod_{i=1}^m P(X_i|\underline{X}_{Pa_i}).
\end{align}
\end{small}Up to some technical conditions \cite{lauritzen1996graphical}, this factorization is equivalent to the \textit{causal Markov} condition.
 Causal Markov condition states that a DAG is only acceptable as a possible causal hypothesis if every node is conditionally independent of its non-descendant given its parents. 
  
 Corresponding DAG of a joint distribution possesses \textit{Global Markov} condition if for any disjoint set of nodes $\mathcal{A}$, $\mathcal{B}$, and $\mathcal{C}$ for which $\mathcal{A}$ and $\mathcal{B}$ are d-separated\footnote{It is d-seperated by $Z$ if it contains a collider $\rightarrow\!\cdot\!\leftarrow$ whose descendants are not in $Z$ or a non-collider in $Z$. }  by $\mathcal{C}$, then $\underline{X}_{\mathcal{A}}\independent\underline{X}_{\mathcal{B}}|\underline{X}_{\mathcal{C}}$.
It is shown in \cite{lauritzen1996graphical} that causal Markov condition and Global Markov condition are equivalent. 
  
\textbf{Faithfulness}:\ 
 A joint distribution is called \textit{faithful} with respect to a DAG if all the conditional independence (CI) relationships implied by the distribution can also be found from its corresponding DAG using d-separation and vice versa\footnote{The set of distributions that do not satisfy this assumption has measure zero \cite{meek1995strong}.}  \cite{pearl2014probabilistic}. 
 It is possible that several DAGs encode the same set of CI relationships. In this case, they are called Markov equivalence. 

 \vspace{-.3cm}
\section{New Dependency Measure}\label{sec:cau}
\vspace{-.1cm}

{
As we mentioned earlier, we use the following paradigm to define our measure of independency: if $Y$ has no dependency on $X$ given $Z$, then the conditional distribution of $Y$ given $X=x$ and $Z=z$ should not change when $X$ takes different realization $x'$ while $Z$ takes the same realization $z$.
This paradigm is similar in nature to Pearl's paradigm of causal influence \cite{pearl2003causality}. He proposed that the influence of a variable (potential cause) on another variable (effect) in a network is assessed by assigning different values to the potential cause, while other variables' effects are removed, and observing the behavior of the effect variable.}
Below, we formally introduce our dependency measure.

Consider $\underline{X}$ a collection of $m$ random variables. In order to identify the dependency of $X_i$ on $X_j$, we select a set of indices $\mathcal{K}$, where $\mathcal{K}\subseteq-\{i,j\}$ and consider the following two probability measures:
\begin{small}
\begin{equation}\label{mui}
  \begin{aligned}
    \mu_i(\underline{x}_{\mathcal{K}\cup\{j\}})\!:=\!P\Big{(}X_i\Big{|}\underline{X}_{\mathcal{K}\cup\{j\}}=\underline{x}_{\mathcal{K}\cup\{j\}}\Big{)},\\
    \mu_i(\underline{y}_{\mathcal{K}\cup\{j\}})\!:=\!P\Big{(}X_i\Big{|}\underline{X}_{\mathcal{K}\cup\{j\}}=\underline{y}_{\mathcal{K}\cup\{j\}}\Big{)},
  \end{aligned}
\end{equation}
\end{small}where $\underline{x}_{\mathcal{K}\cup\{j\}}$ and $\underline{y}_{\mathcal{K}\cup\{j\}}\in E^{|\mathcal{K}|+1}$ are two realizations for $\small{\underline{X}_{\mathcal{K}\cup\{j\}}}$ that are the same every where except at $X_j$. 
Further, assume $\underline{x}_{\mathcal{K}\cup\{j\}}$ at position $X_{j}$ equals $x$ and $\underline{y}_{\mathcal{K}\cup\{j\}}$ equals $y$ ($y\neq x$) at this position. 
If there exists a subset $\mathcal{K}\subseteq-\{i,j\}$ such that for all such realizations, $\mu_i(\underline{x}_{\mathcal{K}\cup\{j\}})$ and $\mu_i(\underline{y}_{\mathcal{K}\cup\{j\}})$ are the same, then we say $X_i$ has no dependency on $X_j$.  
This is analogous to the conditional independence that states if $X_j$ and $X_i$ are independent given some $\underline{X}_{\mathcal{K}}$, then there is no causal influence between them.
Note that using mere observational data, comparing the two conditional probabilities in (\ref{mui}) reveals the dependency between $X_i$ and $X_j$. However, when interventional data is available, we can identify whether $X_j$ causes $X_i$, i.e., the direction of influence.

In order to compare the two probability measure in (\ref{mui}), a metric on the space of probability measures is required. There are several metrics that can be used such as KL-divergence, total variation, etc \cite{gibbs2002choosing}. 
For instance, using the KL-divergence will lead to develop CI test-based approaches \cite{singh1995construction}.
 In this work, we use Wasserstein distance and discuss the advantage of using such metric in Section \ref{sec:com}.
\begin{definition}
 Let $(E, d)$ be a metrical complete and separable space equipped with the Borel field $\mathcal{B}$, and let $\mathcal{M}$ be the space of all probability measures on $(E,\mathcal{B})$. Given $\nu_1,\nu_2\in\mathcal{M}$, the Wasserstein metric between $\nu_1, \nu_2$ is given by
$
W_d(\nu_1,\nu_2):=\inf_{\pi}\left(\mathbb{E}_{\pi}[d(x,y)] \right)
$, where the infimum is taken over all probability measures $\pi$ on $E\times E$ such that its marginal distributions are  $\nu_1$ and $\nu_2$, respectively. 
\end{definition}
Using the above distance, we define the dependency of $X_i$ on $X_j$ given $\mathcal{K}\subseteq-\{i,j\}$ as follows:
\begin{small}
\begin{eqnarray}\label{dobb}
\small{c^{\mathcal{K}}_{i,j}\!:=\!\!\!\!\!\!\!\!\sup_{\underline{x}_{\mathcal{K}\cup\{j\}}=\underline{y}_{\mathcal{K}\cup\{j\}}, \ \text{off}\ j}\!\!\!\!\dfrac{W_d\Big{(}\mu_i(\underline{x}_{\mathcal{K}\cup\{j\}}),\mu_i(\underline{y}_{\mathcal{K}\cup\{j\}})\Big{)}}{d(x,y)}}.
\end{eqnarray}
\end{small}The suprimum is over all realizations $\underline{x}_{\mathcal{K}\cup\{j\}}$ and $\underline{y}_{\mathcal{K}\cup\{j\}}$ that only differ at the $j$th variable. 
 Moreover, we assume $\underline{x}_{\mathcal{K}\cup\{j\}}$ at $j$th position equals $x$ and $\underline{y}_{\mathcal{K}\cup\{j\}}$ equals $y$ ($y\neq x$) at this position. 
 When $\mathcal{K}=-\{i,j\}$, $c^{\mathcal{K}}_{i,j}$ is called Dobrushin's coefficient \cite{dobrushin1970prescribing}.
Similarly, we define the dependency of a set of nodes $\mathcal{B}$ on a disjoint set $\mathcal{A}$ given $\mathcal{K}$, where $\mathcal{K}\cap(\mathcal{A}\cup\mathcal{B})=\emptyset$, as follows,
 \begin{small}
 \begin{eqnarray}\label{dobbset}
\small{c^{\mathcal{K}}_{\mathcal{B},\mathcal{A}}\!:=\!\!\!\!\!\!\!\!\sup_{{\underline{x}_{\mathcal{K}\cup\mathcal{A}}=\underline{y}_{\mathcal{K}\cup\mathcal{A}},\ \text{off}\ \mathcal{A}}}\!\!\!\!\dfrac{W_d\Big{(}\mu_\mathcal{B}(\underline{x}_{\mathcal{K}\cup\mathcal{A}}),\mu_\mathcal{B}(\underline{y}_{\mathcal{K}\cup\mathcal{A}})\Big{)}}{d(\underline{x}_\mathcal{A},\underline{y}_\mathcal{A})}}.
\end{eqnarray}
\end{small}\begin{remark}
An alternative way of interpreting the above measure is via an equivalent network in which all the nodes in the set $\mathcal{K}\cup\{j\}$ are injected with independent inputs that have distributions equal to their marginals, i.e., node $k$ is injected with an independent random variable that has distribution $P(X_k)$. 
In this equivalent network, the dependency of $i$ on $j$ given $\mathcal{K}$ can be expressed by
\begin{small}
\begin{align*}
&\int_E \prod_{k\in\mathcal{K}}P(\underline{X}_k=\underline{x}_k)P(X_j=y)P(X_j=x) \\ \notag
&\dfrac{W_d\Big{(}\mu_i(\underline{x}_{\mathcal{K}\cup\{j\}}),\mu_i(\underline{y}_{\mathcal{K}\cup\{j\}})\Big{)}}{d(x,y)}d\underline{x}_k dxdy.
\end{align*}
\end{small}Clearly, this expression is bounded above by (\ref{dobb}).
\end{remark}

\vspace{-.3cm}
\subsection{Maximum Mean Discrepancy}\label{sec:md}
Using a special case of the duality theorem of Kantorovich and Rubinstein \cite{villani2003topics}, we obtain an alternative approach for computing the Wasserstein metric as follows:
\begin{small}
\begin{equation}\label{dual}
W_d(\nu_1,\nu_2)=\sup_{f\in\mathcal{F}_L}\left\vert\int_{E}fd\nu_1-\int_{E}fd\nu_2\right\vert,
\end{equation}
\end{small}where $\mathcal{F}_L$ is the set of all continuous functions satisfying the Lipschitz condition:
\begin{small}
$
||f||_{\text{Lip}}:=\sup_{x\neq y}|f(x)-f(y)|/d(x,y) \leq1.
$
\end{small}This representation of the Wasserstein metric is a special form of integral probability metric (IPM) \cite{muller1997integral} that has been studied extensively in probability theory \cite{dudley2002real} with applications in empirical process theory \cite{van1996weak}, transportation problem \cite{villani2003topics}, etc. 
IPM is defined similar to (\ref{dual}) but instead of $\mathcal{F}_L$, the suprimum is taken over a class of real-valued bounded measurable functions on $E$.

One particular instance of IPM is maximum mean discrepancy (MMD) in which the suprimum is taken over $\mathcal{F}_\mathcal{H}:=\{f : ||f||_\mathcal{H} \leq 1\}$. More precisely, MMD is defined as
\begin{small}
\begin{align}\label{dobb3}
\text{MMD}(\nu_1,\nu_2):=\sup_{f\in\mathcal{F}_\mathcal{H}}\left\vert\int_{E}fd\nu_1-\int_{E}fd\nu_2\right\vert,
\end{align}
\end{small}Here, $\mathcal{H}$ represents a reproducing kernel Hilbert space (RKHS) \cite{aronszajn1950theory} with reproducing kernel $k(\cdot,\cdot)$. 
MMD has been used in statistical applications such as independence testing and testing for conditional independence \cite{gretton2007kernel,fukumizu2007kernel,sun2007kernel}.

It is shown in \cite{gretton2006kernel} that when $\mathcal{H}$ is a universal RKHS \cite{micchelli2006universal}, defined on the compact metric space $E$, then $\text{MMD}(\nu_1,\nu_2)=0$ if and only if $\nu_1=\nu_2$. 
In this case, MMD can also be used to compare the two conditional distributions in (\ref{mui}). 
This is because, $\text{MMD}(\mu_i(\underline{x}_{\mathcal{K}\cup\{j\}}),\mu_i(\underline{y}_{\mathcal{K}\cup\{j\}}))=0$ implies that the two conditional distributions are the same.
This allows us to define a new dependency measure which we denoted it by $\tilde{c}^{\mathcal{K}}_{i,j}$ similar to (\ref{dobb}) that uses MMD instead of Wasserstein distance.
It is straight forward to show that this measure has similar properties as the one in (\ref{dobb}). The main difference between these two measures is their estimation method that we discuss in Section \ref{sec:com}.

\vspace{-.3cm}
\section{Advantages of the Dependency Measure}\label{sec:disc}
\vspace{-.1cm}
Herein, we discuss the advantages of our measure over other dependency measures in the literature.

\vspace{-.2cm}
\subsection{Mutual Information and Information Flow}\label{sec:info}
\vspace{-.1cm}
Conditional mutual information is an information theoretic measure that has been used in the literature to identify the conditional independence structure of a network. This measure compares two probability measures 
\begin{small}$P(X_i|X_j,\underline{X}_{\mathcal{K}})$\end{small} and \begin{small}$P(X_i|\underline{X}_{\mathcal{K}})$\end{small} using the KL-divergence as follows,
\begin{small}
\begin{align}\label{conmu}
I(X_i;X_j|\underline{X}_{\mathcal{K}}):=\sum_{x_i,x_j,\underline{x}_{\mathcal{K}}}P(x_i,x_j,\underline{x}_{\mathcal{K}})\log\frac{P(x_i|x_j,\underline{x}_{\mathcal{K}})}{P(x_i|\underline{x}_{\mathcal{K}})}.
\end{align}
\end{small}This measure is symmetric and hence it cannot capture the direction of influence. Moreover, it only compares the probability measures over all pairs $(X_i,X_j)$ that have positive probability. 
Note that any other measures in the literature that is based on conditional independence test such as the kernel-based methods in \cite{sun2007kernel,ZhangPJS2011} have the similar limitation.

\begin{example}\label{ex:mu}
Consider a network of two variables $X$ and $Y$, in which $X\sim\mathcal{N}(0,1)$ is a zero mean Gaussian variable and $Y$ is $\mathcal{N}(0,1)$ whenever $X$ is a rational number and $\mathcal{N}(1,2)$ otherwise. In this network, $Y$ is dependent on $X$ but it cannot be captured using CI. This is because $I(X;Y)=0$. On the other hand, we have $c_{y,x}>0$ and $c_{x,y}=0$. 
\end{example}

Another quantity that has been introduced in the literature to quantify causal influences in a network is information flow \cite{ay2008information}. This quantity is defined using Pearl's do-calculus \cite{pearl2003causality}.
 Intuitively, operating $do(x_i)$ removes the dependencies of $X_i$ on its parents, and replaces $P(X_i|\underline{X}_{Pa_i})$ with the delta function. 
 Herein, to give an interpretation on how (\ref{dobb}) can be used to identify causal relationships that are defined in terms of intervention, we compare our measure with information flow.
 
Below, we introduce the formal definition of information flow from $\underline{X}_A$ to $\underline{X}_B$ imposing $\underline{X}_\mathcal{K}$, $I(\underline{X}_A\rightarrow\underline{X}_B|do(\underline{X}_\mathcal{K}))$, where $A$, $B$, and $\mathcal{K}$ are three disjoint subsets of $V$.
\begin{small}
\begin{align}\label{con1}
&\sum_{\underline{x}_{A\cup B\cup \mathcal{K}}}P(\underline{x}_\mathcal{K})P(\underline{x}_A|do(\underline{x}_\mathcal{K}))P(\underline{x}_B|do(\underline{x}_{A\cup \mathcal{K}}))\\ \notag
&\log\frac{P(\underline{x}_B|do(\underline{x}_{A\cup \mathcal{K}}))}{\sum_{\underline{x}'_{A}}P(\underline{x}'_A|do(\underline{x}_\mathcal{K}))P(\underline{x}_B|do(\underline{x}'_A,\underline{x}_\mathcal{K}))}.
\end{align}
\end{small}This is defined analogous to the conditional mutual information in (\ref{conmu}). But unlike the conditional mutual information, the information flow is defined for all pairs $(\underline{x}_A; \underline{x}_C)$ rather than being limited to those with positive probability (similar to our measure).  
Similar measures are introduced in \cite{janzing2013quantifying,ay2007geometric} which are also based on do-calculation. 
Analogously, we can define our measure based on do-operation in order to capture the direction of causal influences in a network by substituting the conditional distributions in (\ref{mui}) with their $do$ versions.


Because the Wasserstein metric can be estimated using a linear programming (see Section \ref{sec:com}), our measure has computational advantages over the information flow or other similar measures that uses KL-divergence.
Another advantage of (\ref{dobb}) over the information flow is that it requires less number of interventions in case of using interventional data. More precisely, calculating (\ref{con1}) requires at least two do-operations $(do(\underline{x}_{A\cup\mathcal{K}})$ and $do(\underline{x}_{\mathcal{K}}))$ but (\ref{dobb}) requires only one $(do(\underline{x}_{\mathcal{K}\cup\{j\}}))$. 
Moreover, as the next example shows, unlike our measure, the information flow depends on the underlying DAG.

\begin{example}
Consider a network of three binary random variables $\{X,Y,Z\}$ with $Z=X\oplus Y$ an XOR. 
Suppose the underlying DAG of this network is given by Figure \ref{ex:Figf}(b), in which $X$ takes zero with probability $b$.
In this case, $I(X\rightarrow Z|do(Y))=H(b)$, where $H$ denotes the entropy\footnote{More precisely, $H(b)=-b\log b-(1-b)\log(1-b)$.}.
However, if the underlying DAG is given by Figure \ref{ex:Figf}(a), we have $I(X\rightarrow Z|do(Y))=H(\epsilon)$. 
Now, consider a scenario in which $\epsilon$ tends to zero. In this scenario, both DAGs describe a system in which $X=Y$ and $Z=0$. However, in (b), we have $I(X\rightarrow Z|do(Y))=H(b)>0$, while in (a), $I(X\rightarrow Z|do(Y))\rightarrow0$.
But $c^y_{z,x}$ in both DAGs is independent of $\epsilon$ and it is positive.
\end{example}

\begin{figure*}
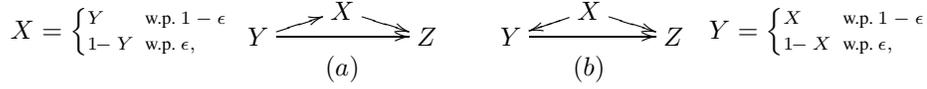

\hspace{1.7cm}
{\xygraph{ !{<0cm,0cm>;<.8cm,0cm>:<0cm,.7cm>::} 
!{(1,.8) }*+{X}="x" 
!{(-2.7,.4) }*+{
\begin{scriptsize}
X=\begin{cases}
Y &\!\! \text{w.p.}\ 1-\epsilon\\
1\!\! -Y &\!\! \text{w.p.}\ \epsilon,
\end{cases}
\end{scriptsize}
}="x22" 
!{(-.4,0.3) }*+{Y}="y"
!{(2.4,0.3)}*+{Z}="z"
!{(1,-.3) }*+{(a)}="" 
!{(5.1,.8) }*+{X}="x1" 
!{(8.9,.4) }*+{
\begin{scriptsize}
Y=\begin{cases}
X &\!\! \text{w.p.}\ 1-\epsilon\\
1\!\! -X &\!\! \text{w.p.}\ \epsilon,
\end{cases}
\end{scriptsize}
}="x212" 
!{(3.8,0.3) }*+{Y}="y1"
!{(6.5,0.3)}*+{Z}="z1"
!{(5.1,-.3) }*+{(b)}="" 
 "x1":"z1"  "y1":"z1"  "x1":"y1"
 "x":"z"  "y":"z"  "y":"x"}\, \, \,}
\caption{DAGs for which information flow fails to capture the influence.}\label{ex:Figf}
\end{figure*}

\vspace{-.3cm}
\subsection{A Better Measure for Direct Causal Influences}\label{sec:direct}
\vspace{-.2cm}
Consider a network comprises of three random variables $\{X,Y,Z\}$, in which $Y=f(X,W_1)$ and $Z=g(X,Y,W_2)$, such that 
the transformations from $(X,W_1)$ to $(X,Y)$ and from $(X,Y,W_1)$ to $(X,Y,Z)$ are invertible and $W_1$ and $W_2$ are independent exogenous noises.  
In other words, there exist functions $\phi$ and $\varphi$ such that $W_1=\phi(X,Y)$ and $W_2=\varphi(X,Y,Z)$.
Furthermore, $f$ is an injective function in its first argument, i.e., if $f(x_1,w)=f(x_2,w)$ for some $w$, then $x_1=x_2$. 

In order to measure the direct influence from $X$ to $Z$, one may compute the conditional mutual information between $X$ and $Z$ given $Y$, i.e., $I(X;Z|Y)$. 
However, this is not a good measure because as the dependency of $Y$ on $X$ grows, i.e., $H(Y|X)\rightarrow 0$, then $I(X;Z|Y)\rightarrow0$.
This can be explained
by the fact that as $H(Y|X)$ goes to zero, in other words, as $P_{W_1}$ tends to $\delta_{w_0}(W_1)$ for some fixed value $w_0$, then by specifying the value of $X$, the ambiguity about the value of $Y$ will go to zero. 
Thus, using the injective property of $f$, it is straight forward to see that $I(X;Z|Y)\rightarrow0$.

This analysis shows that $I(X;Z|Y)$ fails to capture the direct influence between $X$ and $Z$ when $Y$ depends on $X$ almost in a deterministic manner. However, looking at $c^{y}_{z,x}$, we have
\begin{small}
\begin{align*}
\!c^{y}_{z,x}\!=
\!\!\sup_{y,x,x'}\!\frac{W_d\left(P_{x,y}(Z),P_{x',y}(Z)\right)}{d(x,x')},
\end{align*}
\end{small}where \begin{footnotesize}$P_{x,y}(Z):=P_{W_2}(\varphi(x,y,Z))|\frac{\partial g}{\partial W_2}(x,y,\varphi(x,y,Z))|^{-1}\!\!.$ \end{footnotesize}
This distribution depends only on realizations of $(X,Y)$ and it is independent of $P_{X,Y}$. Hence, changing the dependency between $X$ and $Y$ will not affect $c^{y}_{z,x}$, which makes it a better candidate to measure the direct influences between variables of a network. 
 As an illustration, we present a simple example. But first, we need the following result.
 \begin{theorem}\label{coro}
Consider $\overline{X}=\textbf{A}\overline{X}+\overline{W}$, where $\textbf{A}$ has zero diagonals and its support represents a DAG. $\overline{W}$ is a vector of zero mean independent random variables. Then, $c^{Pa_i\setminus\{j\}}_{i,j}=|A_{i,j}|.$
\end{theorem}

 \begin{example}\label{example12}
Consider a network of three variables $\{X,Y,Z\}$ in which $Y=aX+W_1$ and $Z=bX+cY+W_2$ for some non-zero coefficients $\{a,b,c\}$ and exogenous noises $\{W_1,W_2\}$. Hence,  
 \begin{small}
 \begin{align}\label{eq:w4}
 &I(X;Z|Y)=H(bX+W_2| aX+W_1)-H(W_2).
 \end{align}
 \end{small}As we mentioned earlier, by reducing the variance of $W_1$, the first term in (\ref{eq:w4}) tends to $H(bX+W_2|X)=H(W_2)$. Hence, (\ref{eq:w4}) goes to zero.  But, using the result of Theorem \ref{coro}, we have $c^{y}_{z,x}=|b|$, which is independent of the variance of  $W_1$.
 \end{example}

\vspace{-.2cm}
\subsection{Group Selection for Effective Intervention}\label{sec:cc}
\vspace{-.2cm}

Consider a network of three variables $\{X, Y, C\}$ in which $C$ is a common cause for $X$ and $Y$, and $X$ influences $Y$.
In this network, to measure the influence of $X$ on $Y$, one may consider \begin{small}$P(Y|do(X))$\end{small} that is given by \begin{small}$\sum_c P(Y|X,c)P(c)=\mathbb{E}_c[P(Y|X,c)]$\end{small}. See, e.g., the back-door criterion in \cite{pearl2003causality}.
This conditional distribution is an average over all possible realizations of the common cause $C$. 

Consider an experiment that is been conducted on a group of people with different ages $C$ in which the goal is to identify the effect of a treatment $X$ on a special disease $Y$.
Suppose that this treatment has clearer effect on that disease for elderly people and less obvious effect for younger ones.
In this case, averaging the effect of the treatment on the disease for all people with different ages, i.e., \begin{small}$P(Y|do(X))$\end{small} might not reveal the true effect of the treatment. Hence, it is important to identify a regime (in this example age range) of $C$ in which the influence of $X$ on $Y$ is maximized. As a consequence, we can identify the group of subjects on which the intervention is effective.

Note that this problem cannot be formalized using do-operation or other measures that take average over all possible realizations of $C$.
However, using the measure in \eqref{dobb}, we can formulate this problem as follows: given $X=x$ and two different realizations for $C$, say $c$ and $c'$, we obtain two conditional probabilities \begin{small}$P(Y|x,c)$\end{small} and {\small$P(Y|x,c')$}. 
Then, we say in group $C=c$, the causal influence between $X$ and $Y$ is more obvious compare to the group $C=c'$, if given $C=c$, changing the assignments of $X$ leads to larger variation of the conditional probabilities compared to changing the assignment of $X$ given $C=c'$. More precisely, if 
$c_{y,x}^{C=c}\geq c_{y,x}^{C=c'}$, where 
\begin{small}
\begin{align}\label{eq:common}
&c_{y,x}^{C=c}:=\sup_{x\neq x'}\frac{W_d\Big(P(Y|x,c),P(Y|x',c)\Big)}{d(x,x')}.
\end{align}
\end{small}Note that $c_{y,x}^{c}=\sup_{c}c_{y,x}^{C=c}$, where $c_{y,x}^{c}$ is given in \eqref{dobb}. Using this new formulation, we define the range of $C$ in which the influence from $X$ to $Y$ is maximized as  $\arg\max_{c}c_{y,x}^{C=c}$.

\begin{example}
Suppose that $Y=CX+W_2$ and $X=W_1/C$, where $C$ takes value from $\{1,...,M\}$ w.p. $\{p_1,...,p_M\}$ and $W_i\sim\mathcal{N}(0,1)$. In this case, we have $c_{y,x}^{C=c}=|c|$. Thus, $C=M$ will show the influence of $X$ on $Y$ more clearer.
On the other hand, such property cannot be detected using other measures. For example, we have
\begin{small}
$
I(X;Y|C=c)=0.5\log(2),
$ 
\end{small} for all $c$.
\end{example}


\vspace{-.3cm}
\section{Properties of the Measure}\label{sec:pro}
\vspace{-.1cm}
  \begin{lemma}\label{lemma1}
 The measure defined in (\ref{dobb}) possesses the following properties: (1) \textit{Asymmetry}: In general $c^{\mathcal{K}}_{i,j}\neq c^{\mathcal{K}}_{j,i}$. (2) $c^{\mathcal{K}}_{i,j}\geq0$ and when it is zero, we have $X_i\independent  X_j|\underline{X}_{\mathcal{K}}$. (3)  \textit{Decomposition}: $c^{\mathcal{K}}_{i,\{j,k\}}=0$ implies $c^{\mathcal{K}}_{i,j}\!=\!c^{\mathcal{K}}_{i,k}\!=\!0$. (4) \textit{Weak union}: If $c^{\mathcal{K}}_{i,\{j,k\}}\!=0$, then $c^{\mathcal{K}\cup\{k\}}_{i,j}\!=\!c^{\mathcal{K}\cup\{j\}}_{i,k}\!=0$. (5) \textit{Contraction}: If $c^{\mathcal{K}}_{i,j}\!=\!c_{i,{\mathcal{K}}}\!=0$, then $c_{i,{\mathcal{K}}\cup\{j\}}\!=0$. (6) \textit{Intersection}: If $c^{\mathcal{K}\cup\{k\}}_{i,j}\!\!=c^{\mathcal{K}\cup\{j\}}_{i,k}\!\!=0$, then $c^{\mathcal{K}}_{i,\{j,k\}}=0$.
 \end{lemma}
 Note that unlike the intersection property of the conditional independence, which does not always hold, the intersection property of the dependency measure in (\ref{dobb}) always holds. This is due to the fact that (\ref{dobb}) is defined for all realizations $(x_j,\underline{x}_\mathcal{K})$ not only those with positive measure.
See Example \ref{ex:mu} for the asymmetric property of $c^{\mathcal{K}}_{i,j}$.
  
We say a DAG possesses global Markov property with respect to (\ref{dobb}) if for any node $i$ and disjoint sets $\mathcal{B}$, and $\mathcal{C}$ for which $i$ is d-separated from $\mathcal{B}$ by $\mathcal{C}$, we have $c^\mathcal{C}_{i,\mathcal{B}}=c^\mathcal{C}_{\mathcal{B},i}=0$.
Using the above Lemma and the results of Theorem 3.27 in \cite{lauritzen1996graphical}, it is straightforward to show that a faithful network of $m$ random variables whose causal structure is a DAG possesses the global Markov property\footnote{See Appendix for more details.}.
This property can be used to develop reconstruction algorithms (e.g., PC algorithm \cite{spirtes2000causation}) for the causal structure of a network. 



\begin{figure*}
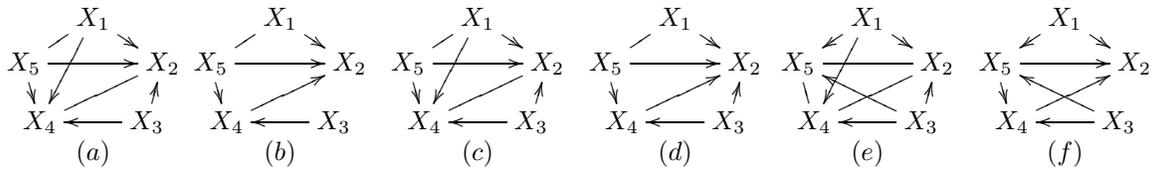

\hspace{1cm}
{\xygraph{ !{<0cm,0cm>;<.71cm,0cm>:<0cm,.6cm>::} 
!{(1,1) }*+{X_1}="x1" 
!{(-.3,0) }*+{X_5}="x5"
!{(2.3,0)}*+{X_2}="x2"
!{(0,-1.3) }*+{X_4}="x4"
!{(2,-1.3)}*+{X_3}="x3"
!{(1,-2) }*+{(a)}
!{(4.5,1) }*+{X_1}="x1a" 
!{(3.2,0) }*+{X_5}="x5a"
!{(5.8,0)}*+{X_2}="x2a"
!{(3.5,-1.3) }*+{X_4}="x4a"
!{(5.5,-1.3)}*+{X_3}="x3a"
!{(4.5,-2) }*+{(b)}
!{(8.2,1) }*+{X_1}="x1b" 
!{(6.9,0) }*+{X_5}="x5b"
!{(9.5,0)}*+{X_2}="x2b"
!{(7.2,-1.3) }*+{X_4}="x4b"
!{(9.2,-1.3)}*+{X_3}="x3b"
!{(8.2,-2) }*+{(c)}
!{(11.9,1) }*+{X_1}="x1c" 
!{(10.6,0) }*+{X_5}="x5c"
!{(13.2,0)}*+{X_2}="x2c"
!{(10.9,-1.3) }*+{X_4}="x4c"
!{(12.9,-1.3)}*+{X_3}="x3c"
!{(11.9,-2) }*+{(d)}
!{(15.5,1) }*+{X_1}="x1d" 
!{(14.2,0) }*+{X_5}="x5d"
!{(16.8,0)}*+{X_2}="x2d"
!{(14.5,-1.3) }*+{X_4}="x4d"
!{(16.5,-1.3)}*+{X_3}="x3d"
!{(15.5,-2) }*+{(e)}
!{(19.2,1) }*+{X_1}="x1e" 
!{(17.9,0) }*+{X_5}="x5e"
!{(20.5,0)}*+{X_2}="x2e"
!{(18.2,-1.3) }*+{X_4}="x4e"
!{(20.2,-1.3)}*+{X_3}="x3e"
!{(19.2,-2)}*+{(f)}
"x1":"x2" "x5":"x2" "x4"-"x2" "x5":"x4" "x1"-"x5" "x3":"x2" "x1":"x4" "x3":"x4"
"x1a":"x2a" "x5a":"x2a" "x3a":"x4a" "x5a":"x4a" "x4a":"x2a" "x1a"-"x5a"
"x1b":"x2b" "x5b":"x2b" "x4b"-"x2b" "x5b":"x4b" "x1b"-"x5b" "x3b":"x2b" "x1b":"x4b" "x3b":"x4b"
"x1c":"x2c" "x5c":"x2c" "x3c":"x4c" "x5c":"x4c" "x4c":"x2c" "x1c"-"x5c" "x3c":"x2c"
"x1d":"x2d" "x1d":"x5d" "x1d":"x4d" "x3d":"x5d" "x3d":"x2d" "x3d":"x4d" "x4d"-"x5d" "x2d"-"x5d" "x2d"-"x4d"
"x1e":"x2e" "x5e":"x2e" "x3e":"x4e" "x5e":"x4e" "x4e":"x2e" "x1e":"x5e" "x3e":"x5e"
}\, \, \,}
\caption{Recovered DAGs of the system given in (\ref{sim:1}) for different sample sizes. (a)-(b) use the measure in (\ref{dobb}) and pure observation. (c)-(d) use kernel-based method and pure observation. (e)-(f) use the measure in (\ref{dobb}) and interventional data. (f) shows the true structure.}\label{fig:sim}
\end{figure*}

\vspace{-.2cm}
\subsection{Estimation}\label{sec:com}
\vspace{-.1cm}
The measure introduced in (\ref{dobb}) can be computed explicitly for special probability measures. For instance, if the joint distribution of $\underline{X}$ is Gaussian with mean $\vec{\mu}$ and covariance matrix $\Sigma$, then using the results of \cite{givens1984class},
 we obtain 
\begin{small}
$
c^{\mathcal{K}}_{i,j}=|\Sigma_{i,\{j,\mathcal{K}\}}\!\left(\Sigma_{\{j,\mathcal{K}\},\{j,\mathcal{K}\}}\right)^{-1}\!\!\!\textbf{e}_1|,
$
\end{small}where $\Sigma_{i,\{j,\mathcal{K}\}}$ denotes the sub-matrix of $\Sigma$ comprising row $i$ and columns $\{j,\mathcal{K}\}$, and $\textbf{e}_1=(1,0,...,0)^T$.  Hence, in such systems, one can estimate the dependency measure by estimating the covariance matrix.
However, this is not the case in general. Therefore, we introduce a non-parametric method for estimating our dependency measure using kernel method.

Given $\{x^{(1)},...,x^{(N_1)}\}$ and $\{x^{(N_1+1)},...,x^{(N_1+N_2)}\}$ that are i.i.d. samples drawn randomly from $\nu_1$ and $\nu_2$, respectively, the estimator of (\ref{dual}) is given by \cite{sriperumbudur2010non},
\begin{small}
\begin{equation}\label{estim}
\begin{aligned}
&\widehat{W}_d(\hat{\nu}_1,\hat{\nu}_2):=\max_{\{\alpha_i\}}\frac{1}{N_1}\sum_{i=1}^{N_1}\alpha_i-\frac{1}{N_2}\sum_{j=1}^{N_2}\alpha_{j+N_1},
\end{aligned}
\end{equation}
\end{small}such that $|\alpha_i-\alpha_j|\leq d(x^{(i)},x^{(j)}),\  \forall i,j.$
In this equation, $\hat{\nu}_1$ and $\hat{\nu}_2$ are empirical estimator of $\nu_1$ and $\nu_2$, respectively.
The estimator of MMD is given by 
\begin{small}
\begin{equation}\label{estim23}
\begin{aligned}
&(\widehat{\text{MMD}}(\hat{\nu}_1,\hat{\nu}_2))^2:={\sum_{i,j=1}^{N_1+N_2}y_i y_j k(x^{(i)},x^{(j)})},
\end{aligned}
\end{equation}
\end{small}where $y_i:=1/N_1$ for $i\leq N_1$ and $y_i:=-1/N_2$, elsewhere. $k(\cdot,\cdot)$ represents the kernel of $\mathcal{H}$.
It is shown in \cite{sriperumbudur2010non} that (\ref{estim}) converges to (\ref{dual}) as $N_1, N_2\rightarrow\infty$ almost surely as long as the underlying metric space is totally bounded. 
It is important to mention that the estimator in (\ref{estim}) depends on $\{x^{(j)}\}$s only through the metric $d(\cdot,\cdot)$, and thus its complexity is independent of the dimension of $x^{(i)}$, unlike the KL-divergence estimator \cite{wang2005divergence}. 
The estimator in (\ref{estim23}) also converges to (\ref{dobb3}) almost surely with the rate of order $\mathcal{O}(1/\sqrt{N_1}+1/\sqrt{N_2})$, when $k(\cdot,\cdot)$ is measurable and bounded. 

Consider a network of $m$ random variables $\underline{X}$. 
Given $N$ i.i.d. realizations of $\underline{X}$, $\{\underline{z}^{(1)},...,\underline{z}^{(N)}\}$, where $\underline{z}^{(l)}\in E^m$, we use (\ref{estim}) and define
\begin{small}
\begin{align}\label{estdob}
&{\widehat{c}^\mathcal{K}_{i,j}:=\max_{1\leq l,k\leq N}\dfrac{\widehat{W}_d\Big{(}\hat{\mu}_i\left(\underline{z}^{(l)}_{\mathcal{K}\cup\{j\}}\right),\hat{\mu}_i\left(\underline{z}^{(k)}_{\mathcal{K}\cup\{j\}}\right)\Big{)}}{d(z_j^{(l)},z_j^{(k)})}}, 
\end{align}
\end{small}such that $\underline{z}^{(l)}_{\mathcal{K}\cup\{j\}}=\underline{z}^{(k)}_{\mathcal{K}\cup\{j\}}\ \text{off} \ j$.
Similarly, one can introduce an estimator for $\tilde{c}^{\mathcal{K}}_{i,j}$ using (\ref{estim23}). By applying the result of Corollary 5 in \cite{spirtes2000constructing}, we obtain the following result.

\begin{corollary}
Let $(E,d)$ be a totally bounded metric space and a network of random variables with positive probabilities, then $\widehat{c}^\mathcal{K}_{i,j}$ converges to $c^\mathcal{K}_{i,j}$ almost surely as $N$ goes to infinity.
\end{corollary}
\vspace{-.2cm}
\vspace{-.2cm}
{
\section{Experimental Results}\label{sec:simu}
\vspace{-.2cm}
Herein, we present two simulations in order to verify the theoretical results. In particular, the first experiment verifies the group selection advantages and the second one shows an application of the measure for capturing rare dependencies.
\\
\textbf{Group selection for :}\ \
In this simulation, we considered a group of individuals ($C\in$\{male,female\}) to study the effect of an special treatment $X$ on their health condition $Y$. For instance, $X$ can denote sleep aids and $Y$ can represent the individual's awareness level in the next morning. Most psychotropic drugs are metabolized in the liver. Because the male body breaks down Ambien and other sleep aids faster, women typically have more of the drug in their system the next morning.
For this simulation, we considered a mathematical model between $X,Y$, and $C$ as follows: $X=\mathcal{N}(1.5,1)$ and $Y=2X+\mathcal{N}(0,1)$, when $C=$female and $X=\mathcal{N}(1,4)$ and $Y=3X+\mathcal{N}(0,9)$, otherwise. 
\begin{figure}
	\label{fig::synloss}
	\hspace{-.4cm}
	\includegraphics[width=1.1\linewidth]{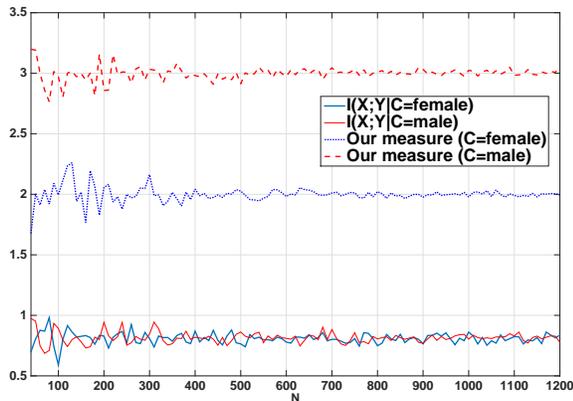}
	\caption{Estimated measures for different $N$.}\label{fig:c}
\end{figure}
Accordingly, we generated different sample sizes $N\in\{40,...,1200\}$ and estimated $I(X;Y|c)$ and $\hat{c}_{y,x}^c$. Figure \ref{fig:c} depicts the results. Since for given $c$, $(X,Y)$ is jointly Gaussian, we estimated $I(X;Y|c)$ by estimating the covariance matrix \cite{cover2012elements}, and estimated our measure using (\ref{estim23}) with Gaussian kernels. As Figure \ref{fig:c} shows, although the treatment has different effects on different genders, $I(X;Y|C)$ cannot capture that.

\textbf{Capturing rare dependencies:}\ \
We simulated the following non-linear system with \begin{small}$W_i\sim U[-1,1]$\end{small} and learned its corresponding structure.
\begin{small}
\begin{align}\notag
&X_1=W_1,\ \ X_2=X_1^2+2X_4-|X_5|+W_2, \ \ X_3=W_3, \\ \label{sim:1}
& X_4=X_3-X_5+W_4, \ \ X_5=W_5, \ \text{if}\ X_3\ \text{is natural}, \ \\ \notag &X_5=2\sqrt{|X_1|}+W_5, \ \text{o.t}.
\end{align}
\end{small}
 We used the estimator of MMD given in (\ref{estim23}) with Gaussian kernels and estimated the dependency measures. 
We obtained the corresponding DAG of this network given a set of observation of size ${\footnotesize N\in\{900,2500\}}$. 
Using the results on the convergence rate of the MMD estimator, we used a threshold of order $\mathcal{O}(1/\sqrt{N})$ to distinguish positive and zero measure.
Figure \ref{fig:sim} depicts the resulting DAGs.
We also compared the performance of our measure with the kernel-based method proposed in \cite{ZhangPJS2011}.
Note that in this example, since the influence of $X_3$ on $X_5$ is not detectable by mere observation, the best we can learn from mere observation is the DAG presented in Figure \ref{fig:sim}(b). 
However, with the same number of observations, the kernel-based method identifies an extra edge, Figure \ref{fig:sim}(d).

Next, we fixed the value of $X_3$ to be natural number and irrational, separately and observed the outcome of the other variables for different sample sizes. Figures \ref{fig:sim}(e)-(f) depict the outcomes of the learning algorithm that uses our measure.
In this case,  $X_3\rightarrow X_5$ was identified and then the Meek rules helped to detect all the directions even the direction of $X_1-X_5$ as it is shown in Figure \ref{fig:sim}(f).
}



\bibliographystyle{named}
\bibliography{ref}

\newpage

\section{Appendix}

\subsection{Proof of Lemma \ref{lemma1}}\label{p:lemma1}
\noindent$\bullet$ $c^{\mathcal{K}}_{i,j}\geq0$ since Wasserstein is a metric. If $c^{\mathcal{K}}_{i,j}=0$, we have 
$
W_d\left(P(X_i|x_j,\underline{x}_{\mathcal{K}}),P(X_i|y_j,\underline{x}_{\mathcal{K}})\right)=0,
$
for all realizations $x_j, y_j$ and $\underline{x}_{\mathcal{K}}$. Using the fact that Wasserstein is a metric on the space of probability measures, the above equality, and total probability law, we obtain 
\begin{small}
\begin{align*}
&P(X_i|\underline{x}_{\mathcal{K}})=\sum_{x_j}P(X_i|x_j,\underline{x}_{\mathcal{K}})P(x_j|\underline{x}_{\mathcal{K}})\\
&=P(X_i|y_j,\underline{x}_{\mathcal{K}})\sum_{x_j}P(x_j|\underline{x}_{\mathcal{K}})=P(X_i|y_j,\underline{x}_{\mathcal{K}}).
\end{align*}
\end{small}
The above equality holds for all $y_j$ and $\underline{x}_\mathcal{K}$. This implies $X_i\independent X_j|\underline{X}_\mathcal{K}$.

\noindent$\bullet$ We show this by an example. Let $X=U_{[0,1]}$ to be uniformly distributed between zero and one, and 
$$
Y=\begin{cases}
V_{[0,1]} & \text{if}\  X\in \mathcal{A},\\
U_{[0,1]} & \text{otherwise},
\end{cases}
$$
where $\mathcal{A}=\{\frac{i}{i+1}: i\in\mathbb{N}\}$, and $V_{[0,1]}$ is a random variable independent of $U$ that is distributed non-uniformly over $[0,1]$. In this case, we have 
\begin{small}
\begin{align*}
c_{y,x}\geq\frac{W_d(P(Y|X=1/2),P(Y|X=\sqrt{2}))}{d(1/2,\sqrt{2})}>0.
\end{align*}
\end{small}
On the other hand, it is easy to see that $Y$ has a uniform distribution over $[0,1]$ almost surely. Furthermore, for two measurable sets $C$ and $B$ in the $\sigma$-algebra, we have
\begin{small}
\begin{align*}
&\!\! P(X\!\!\in C|Y\!\!\in B)=\frac{P(Y\!\!\in B|X\!\!\in C)P(X\!\!\in C)}{P(Y\!\!\in B)}=\\
&\!\! \!\! \frac{P(Y\!\!\in\!\! B|X\!\!\in\!\! C\!\! \cap\!\! \mathcal{A})P(X\!\!\in\!\! C\!\! \cap\!\! \mathcal{A})+P(Y\!\!\in\!\! B|X\!\!\in\!\! C\!\! \setminus\!\! \mathcal{A})P(X\!\!\in\!\! C\!\! \setminus\!\! \mathcal{A})}{P(Y\in B)}\\
&\!\!\! \!=\frac{P(Y\!\!\in\!\! B|X\!\!\in\!\! C\!\! \setminus\!\! \mathcal{A})P(X\!\!\in\!\! C\!\! \setminus\!\! \mathcal{A})}{P(Y\in B)}=P(X\!\!\in\!\! C\!\! \setminus\!\! \mathcal{A}).
\end{align*}
\end{small}
The last equality uses the fact that $P(Y\in B)=P(Y\!\!\in\!\! B|X\!\!\not\in\!\mathcal{A})=P(Y\!\!\in\!\! B|X\!\!\in\!\! C\!\! \setminus\!\! \mathcal{A})$. 
Thus, changing the value of $Y$ will not affect the conditional distribution of $X$ given $Y$, i.e., $c_{x,y}=0$.

\noindent$\bullet$ If $c^\mathcal{K}_{i,\{j,k\}}=0$, 
$
W_d(P(X_i|x_j,x_k,\underline{x}_{\mathcal{K}}),P(X_i|y_j,y_k,\underline{x}_{\mathcal{K}}))=0,
$
for all realization $x_j,y_j,x_k,y_k,\underline{x}_{\mathcal{K}}$. By the total probability law, we obtain
\begin{align*}
&P(X_i|x_k,\underline{x}_{\mathcal{K}})=\sum_{x_j}P(X_i|x_j,x_k,\underline{x}_{\mathcal{K}})P(x_j|x_k,\underline{x}_{\mathcal{K}})\\
&\!\!\!=P(X_i|y_j,y_k,\underline{x}_{\mathcal{K}})\sum_{x_j}P(x_j|x_k,\underline{x}_{\mathcal{K}})=P(X_i|y_j,y_k,\underline{x}_{\mathcal{K}}).
\end{align*}
This implies that $P(X_i|x_k,\underline{x}_{\mathcal{K}})=P(X_i|y_j,y_k,\underline{x}_{\mathcal{K}})=P(X_i|y_k,\underline{x}_{\mathcal{K}})$. Hence, $c^{\mathcal{K}}_{i,k}=0$. Similarly, we can prove that $c^{\mathcal{K}}_{i,j}=0$.

\noindent$\bullet$ Suppose $c^\mathcal{K}_{i,\{j,k\}}=0$, then from the previous proof, we have $P(X_i|x_k,\underline{x}_{\mathcal{K}})=P(X_i|y_k,y_j,\underline{x}_{\mathcal{K}})$, for all realizations $y_j,x_k,y_k,\underline{x}_{\mathcal{K}}$. 
Thus, $P(X_i|x_k,\underline{x}_{\mathcal{K}})=P(X_i|y_k,x_j,\underline{x}_{\mathcal{K}})$
This is equivalent to say $c^{\mathcal{K}\cup\{j\}}_{i,k}=0$. The other part can be shown similarly.

\noindent$\bullet$ If $c^{\mathcal{K}}_{i,j}=c_{i,\mathcal{K}}=0$, then from $c^{\mathcal{K}}_{i,j}=0$ and total probability law, we obtain that 
\begin{equation}\label{eses}
W_d(P(X_i|x_j,\underline{x}_{\mathcal{K}}),P(X_i|\underline{x}_{\mathcal{K}}))=0.
\end{equation}
On the other hand, using the triangle inequality of the Wasserstein metric, we have 
\begin{align*}
&W_d(P(X_i|x_j,\underline{x}_{\mathcal{K}}),P(X_i|y_j,\underline{y}_{\mathcal{K}}))\leq\\
&W_d(P(X_i|x_j,\underline{x}_{\mathcal{K}}),P(X_i|\underline{x}_{\mathcal{K}}))+W_d(P(X_i|\underline{x}_{\mathcal{K}}),P(X_i|\underline{y}_{\mathcal{K}}))\\
&+W_d(P(X_i|\underline{y}_{\mathcal{K}}),P(X_i|y_j,\underline{y}_{\mathcal{K}})).
\end{align*}
The first and third expressions on the right hand side are zero due to (\ref{eses}) and the second expression is zero due to $c_{i,\mathcal{K}}=0$.

\noindent$\bullet$ If $c^{\mathcal{K}\cup\{k\}}_{i,j}\!\!=0$, 
$
W_d(P(X_i|x_j,x_k,\underline{x}_\mathcal{K}),P(X_i|y_j,x_k,\underline{x}_\mathcal{K}))=0.
$ 
This implies that $P(X_i|x_j,x_k,\underline{x}_\mathcal{K})=P(X_i|x_k,\underline{x}_\mathcal{K})$ for all realizations $x_j, x_k$, and $\underline{x}_\mathcal{K}$.
Similarly, because of $c^{\mathcal{K}\cup\{j\}}_{i,k}\!\!=0$, we have $P(X_i|x_j,x_k,\underline{x}_\mathcal{K})=P(X_i|x_j,\underline{x}_\mathcal{K})$ for all realizations $x_j, x_k$, and $\underline{x}_\mathcal{K}$. Hence, for all realizations, we have
$
P(X_i|x_j,\underline{x}_\mathcal{K})=P(X_i|x_k,\underline{x}_\mathcal{K}).
$
This result and the total probability law will establish the result.

\subsection{The Global Markov Property}\label{p:theorem1}
Since the influence structure of this network is a DAG, there exists an ordering of the variables such that for every node $i$, all its parents have indices less that $i$. Without loss of generality suppose that $\{X_1,...,X_m\}$ is that ordering. Furthermore, using the chain rule, we have 
\begin{equation}\label{eq:lem:1}
P(\underline{X})=\prod_{i=1}^m P(X_i|\underline{X}_{\{<i\}}),
\end{equation}
where $\underline{X}_{\{<i\}}$ denotes all the variables with indices less than $i$. Due to the nature of this ordering, all the nodes in $\{<i\}$ that do not belong to $Pa_i$ are non-descendants of node $i$. Hence, by the definition of ID, they have zero influence on $X_i$ given the parents of $i$ and because of the first property in Lemma \ref{lemma1}, they can be dropped from the conditioning in (\ref{eq:lem:1}).
\\
The global Markov property is a direct consequence of Lemma \ref{lemma1} and Theorem 3.27 in \cite{lauritzen1996graphical}.

\subsection{Proof of Theorem \ref{coro}}\label{p:coro}
In order to complete the proof, we need the following technical lemmas.
When $d(\cdot,\cdot)$ is the Euclidean distance, we denote the Wasserstein metric by $W_E(\cdot,\cdot)$.
\begin{lemma}\label{pp3}
For real-valued random variables, we have
\begin{eqnarray}\label{ineq23}
&\left|\mathbb{E}_{\nu_1}[x]-\mathbb{E}_{\nu_2}[y]\right|\quad\leq & W_E(\nu_1,\nu_2)\\ \nonumber
&&\hspace{-.6cm}\leq\sqrt{\mathbb{E}_{\nu_1}[x^2]+\mathbb{E}_{\nu_2}[y^2]-2\mathbb{E}_{\pi}[xy]},
\end{eqnarray}
where $\pi$ is any joint distribution of $x$ and $y$ such that its marginals are $\nu_1$ and $\nu_2$.
\end{lemma}
\vspace{-.2cm}
\begin{proof}
The lower bound is due to the dual representation of the Wasserstein metric and the fact that $f(x)=x$ is Lipschitz.\\
For the upper bound, we use
the Jensen's inequality, that is
\begin{equation}\label{ineq}
W_d(\nu_1,\nu_2)\leq \inf_{\pi}\left(\mathbb{E}_{\pi}[d^{p}(x,y)] \right)^{1/p},
\end{equation}
for $p\geq1$. For $p=2$, we use the monotonicity of $\sqrt{x}$, and the fact that the space of probability measures is complete and obtain the result. 
\end{proof}

Consider a network of variables in which every variable $X_i$ functionally depends on a subset of other variables $\underline{X}_{Fp_i}$ (the parent set of node $i$) as follows,
\begin{small}
\begin{equation}\label{ga}
X_i\!=\!F_i(\underline{X}_{Fp_i})\!+\!G_i(\underline{X}_{Fp_i})W_i,\ \ \forall i,
\end{equation}
\end{small}
where $F_i,G_i$ are arbitrary functions such that $G_i\neq0$. $\{W_i\}$s denote exogenous noises with mean zero.

\begin{lemma}\label{pp2}
For a system described by (\ref{ga}), the influence of node $j$ on its child $i$ given the rest of $i$'s parents $Fp_i\setminus\{j\}$ under Euclidean metric, is bounded as follows
\begin{small}
\begin{align}\label{quan}\notag
&\sup_{\substack{\overline{x}_{Fp_i}=\overline{y}_{Fp_i}\\ \text{off}\ j}} \Big{|}\frac{F_{i}(\overline{x}_{Fp_i})-F_{i}(\overline{y}_{Fp_i})}{x-y}\Big{|}\leq c^{Fp_i\setminus\{j\}}_{i,j}\leq \sup_{\substack{\overline{x}_{Fp_i}=\overline{y}_{Fp_i}\\ \text{off}\ j}} \\ 
&\!\!\!\!\left[\!\left(\frac{F_{i}(\overline{x}_{Fp_i})-F_{i}(\overline{y}_{Fp_i})}{x-y}\right)^2\!\!\!\!+\!\!\left(\frac{G_{i}(\overline{x}_{Fp_i})-G_{i}(\overline{y}_{Fp_i})}{x-y}\sigma_{i}\right)^2\right]^{1/2}\!\!\!\!\!\!\!\!.
\end{align}
\end{small}
where the suprimum is taking over all realizations of $\underline{X}_{-\{i\}}$ that are only different at $X_j$.
\end{lemma}
\begin{proof}
Using the lower bound in Lemma \ref{pp3} and the fact that $W_i$s have zero mean, we obtain the lower bound in (\ref{quan}).
\\
To obtain the upper bound, we again use the result of Lemma \ref{pp3}, with the following joint distribution $\pi(X_i,Y_i)$,
\begin{small}
\begin{align*}\notag
\frac{1}{|G_i(\overline{x}_{Fp_i})|} f_{W_i}\left(\Theta_{\overline{x}_{Fp_i}}(X_i)\right)\mathbb{I}_{\{\Theta_{\overline{x}_{Fp_i}}(X_i)=\Theta_{\overline{y}_{Fp_i}}(Y_i)\}},
\end{align*}
\end{small}
where 
\begin{small}
$
\Theta_{\overline{x}_{Fp_i}}(X_i):=\frac{X_i-F_i(\overline{x}_{Fp_i})}{G_i(\overline{x}_{Fp_i})},
$
\end{small}
and $f_{W_i}$ denotes the probability density function of $W_i$ and $\mathbb{I}$ denotes the indicator function. Using this joint distribution, we obtain the upper bound in (\ref{quan}). 
\end{proof}
Applying the above result to a linear system in which $F_i(\overline{y}_{Fp_i})=(\textbf{A}\overline{x})_i$ and $G_i(\overline{x}_{Fp_i})=1$, we obtain that $c_{i,j}^{Fp_i\setminus\{j\}}=|A_{i,j}|$.

\end{document}